\documentclass[doublecolumn,10pt]{elsarticle}

\usepackage[utf8]{inputenc}
\usepackage[T1]{fontenc}
\usepackage[margin=0.75in]{geometry}
\usepackage{amsmath,amssymb,amsthm,bm}
\usepackage{graphicx}
\usepackage{booktabs}
\usepackage{array}
\usepackage{xcolor}
\usepackage{subcaption}
\usepackage{caption}
\usepackage{algorithm}
\usepackage{algpseudocode}
\biboptions{sort&compress}
\usepackage{tikz}
\usepackage{pgfplots}
\pgfplotsset{compat=1.17}
\usetikzlibrary{arrows.meta,positioning,fit,backgrounds,calc}
\usepgfplotslibrary{fillbetween}
\usepackage[colorlinks=true,allcolors=blue!60!black]{hyperref}

\theoremstyle{plain}
\newtheorem{theorem}{Theorem}
\newtheorem{lemma}{Lemma}
\newtheorem{corollary}{Corollary}
\newtheorem{proposition}{Proposition}
\theoremstyle{definition}
\newtheorem{assumption}{Assumption}
\theoremstyle{remark}
\newtheorem{remark}{Remark}

\definecolor{cSrc}{RGB}{31,119,180}
\definecolor{cTgt}{RGB}{214,120,20}
\definecolor{cMod}{RGB}{44,160,44}
\definecolor{cAcc}{RGB}{200,30,30}

\newcommand{\E}{\mathbb{E}}
\newcommand{\soft}{\sigma}
\newcommand{\kap}{\kappa}
\newcommand{\Tstar}{T^{\star}}

\journal{Neural Networks}
\begin{document}

\begin{frontmatter}

\title{When does distribution shift break graph neural networks calibration?}

\author[inst1]{Abderaouf Bahi}
\address[inst1]{Computer Science and Applied Mathematics Laboratory (LIMA),
Faculty of Science and Technology, Chadli Bendjedid University, El Tarf 36000, Algeria}

\begin{abstract}
\noindent
Graph neural networks (GNNs) are increasingly deployed in real-world applications where distribution shift is unavoidable. However, how such shifts affect model calibration, defined as the agreement between predictive confidence and actual accuracy, remains poorly understood, and existing graph calibration methods typically rely on labeled validation data from the deployment distribution. In this work, I present the first closed-form theoretical characterization of GNN calibration under distribution shift. I show that calibration is governed by a single scalar quantity that explicitly depends on structural changes between the source and target graphs, as well as feature quality. This characterization precisely identifies when a model becomes over-confident, under-confident, or remains calibrated, and directly yields the optimal temperature scaling strategy. I further extend the analysis to graph convolutional networks with symmetric normalization, multi-class classification, and covariate shift, and derive a theoretical upper bound on the expected calibration error. My analysis also reveals that, under homogeneous distribution shift, a single global temperature is theoretically optimal, providing a principled explanation for why more complex node-wise recalibration methods offer no additional benefit. Building on these theoretical insights, I propose STAC, a source-free, label-free calibration method. Experiments on synthetic benchmarks demonstrate substantial calibration improvements, while evaluations on five real-world graph datasets show that reliable calibration without target labels remains challenging despite the strong predictive power of the theory. These findings identify label-free accuracy estimation under distribution shift as the central unresolved challenge for practical calibration of GNNs and establish a theoretical foundation for future research on trustworthy graph learning.
 Project details are available at \url{https://github.com/AraoufBh/GNN-Calibration}.
\end{abstract}


\end{frontmatter}

\section{Introduction}\label{sec:intro}

Graph neural networks (GNNs) \cite{bahi2026graphneuralnetworksapplications} are now routinely deployed in settings where the graph seen
at test time differs from the one used for training: social and financial networks
evolve, sensor and infrastructure graphs are rewired after failures or upgrades, and
citation, molecular or recommendation graphs used at inference are drawn from different
sub-populations than those used to train the model. In many of these settings raw
predictive accuracy is not the primary bottleneck; what determines whether a GNN can be
trusted inside a fraud-review, clinical-triage, or infrastructure-monitoring pipeline is
whether its confidence scores are \emph{calibrated}: whether a prediction reported with
$90\%$ confidence is, in fact, correct $90\%$ of the time. Miscalibration is not
cosmetic: it silently corrupts every downstream mechanism that consumes probability
outputs, from selective prediction and human-in-the-loop triage to risk-weighted
decision rules and model ensembling.

Two observations make GNN calibration under distribution shift both practically
important and theoretically awkward. \emph{(i)} GNNs are miscalibrated in
structure-dependent ways; unlike image classifiers (typically over-confident), they are
often \emph{under}-confident in-distribution, and the size of the gap correlates with
graph properties such as homophily and node degree~\citep{wang2021cagcn,hsu2022gats}.
\emph{(ii)} Standard calibrators need labels from the test distribution: temperature
scaling~\citep{guo2017calibration} and its graph-specific
variants~\citep{wang2021cagcn,hsu2022gats,zhuang2025gets,li2025wats,tang2024simcalib} all
fit a temperature (or a temperature-valued function of the node) by minimizing negative
log-likelihood on a labeled validation set drawn from the \emph{same} distribution as
training. Under shift that temperature is simply wrong, and no target labels are
available to refit it.

Despite an increasingly rich empirical literature on GNN calibration (Section~\ref{sec:related}),
no existing account explains, from first principles, \emph{how} and \emph{why} shift
moves a GNN's confidence away from its accuracy, \emph{in which direction}, and \emph{by
how much}. Existing post-hoc graph calibrators (CaGCN, GATS, GETS, WATS,
SimCalib) are all validated in-distribution and are, by construction, vulnerable to
exactly the failure mode that matters most in deployment: a validation set that no
longer represents the test distribution. Empirical studies of GNN reliability under
shift~\citep{trivedi2024gduq} document that the problem is real and that architectural
expressivity alone does not fix it, but stop short of a mechanism. I supply that
mechanism for a model simple enough to analyze in closed form yet expressive enough to
reproduce, quantitatively, the calibration behaviour of trained GCNs in simulation and on
five real graphs spanning two orders of magnitude in edge homophily
(Fig.~\ref{fig:concept} previews the core result).

\paragraph{Contributions.}
\textbf{(C1) A closed-form calibration slope} (Section~\ref{sec:theory}):
$\text{conf}(t)=\soft(t)$ but $\text{acc}(t)=\soft(\kap t)$, with $\kap$ explicit in
homophily and signal-to-noise ratio (SNR); over/under-confidence and $\Tstar=1/\kap$
follow immediately, and I prove a matching ECE bound.
\textbf{(C2) Extensions} to the self-loop GCN operator, to $K$ classes via a signal
coefficient $c_K(h)$, and to covariate shift.
\textbf{(C3) A structural corollary} (global-temperature optimality under homogeneous
shift) that \emph{predicts}, ahead of running any experiment, a negative result: per-node
recalibration cannot help unless shift is heterogeneous across nodes.
\textbf{(C4) A method, a reduction, and an honest negative result}
(Sections~\ref{sec:method}--\ref{sec:exp}): I turn the theory into STAC, a source-free,
label-free recalibrator, and show the whole task reduces to estimating one scalar
(target accuracy) from unlabeled data; a method built on standard estimators halves the
gap on synthetic shifts but is unreliable across five real graphs, while a single oracle
temperature recovers calibration on all of them. Label-free accuracy estimation under
shift is the sole remaining bottleneck.
\textbf{(C5) A positioning of the theory against the closest empirical alternatives}
(Section~\ref{sec:related}, Table~\ref{tab:related}): I situate the closed-form account
against recent empirical GNN-reliability methods~\citep{trivedi2024gduq,huang2023conformal}
and theory-grounded graph domain adaptation~\citep{you2023gda}, clarifying which notion of
reliability each provides and how they could be combined.

The remainder of this paper is organized as follows: Section~\ref{sec:related} positions this work against the
calibration, distribution-shift, and label-free-estimation literatures and summarizes
eleven closely related methods along the axes that matter for deployment
(Table~\ref{tab:related}). Section~\ref{sec:setup} fixes notation and the CSBM I
analyze, and motivates why a linear GNN on a CSBM is the right level of abstraction.
Section~\ref{sec:theory} states and proves the calibration-slope theorem and its
extensions. Section~\ref{sec:method} turns the theory into STAC, a source-free,
label-free recalibration method. Section~\ref{sec:exp} validates every theoretical claim
numerically and evaluates STAC on synthetic shifts and five real graphs.
Section~\ref{sec:discussion} discusses the implications, is honest about limitations, and
lays out the open problems the theory exposes. Section~\ref{sec:conclusion} concludes.
Appendix~\ref{app:proofs} gives full proofs of every theoretical result.

\begin{figure}[t]
\centering
\begin{tikzpicture}
\begin{axis}[width=\linewidth,height=4.7cm,xlabel={logit magnitude $t$},
  ylabel={probability},domain=0:6,samples=80,ymin=0.45,ymax=1.03,
  axis lines=left,xtick={0,2,4,6},ytick={0.5,0.75,1.0},
  legend style={at={(0.98,0.03)},anchor=south east,font=\footnotesize,draw=none,fill=none}]
\addplot[name path=conf,very thick,cSrc]{1/(1+exp(-x))};
\addplot[name path=acc,very thick,cAcc,dashed]{1/(1+exp(-0.5*x))};
\addplot[cSrc!18,forget plot]fill between[of=conf and acc];
\legend{{confidence $\soft(t)$},{accuracy $\soft(\kap t)$, $\kap{=}0.5$}}
\end{axis}
\end{tikzpicture}
\caption{\textbf{The calibration slope.} A source-calibrated linear GNN reports
confidence $\soft(t)$ (blue) but its true accuracy is $\soft(\kap t)$ (red). The shaded
area is the calibration error. When homophily/covariate shift makes $\kap<1$ the model is
over-confident; a single temperature $\Tstar=1/\kap$ (i.e. $\soft(t/\Tstar)=\soft(\kap
t)$) removes the gap exactly.}
\label{fig:concept}
\end{figure}
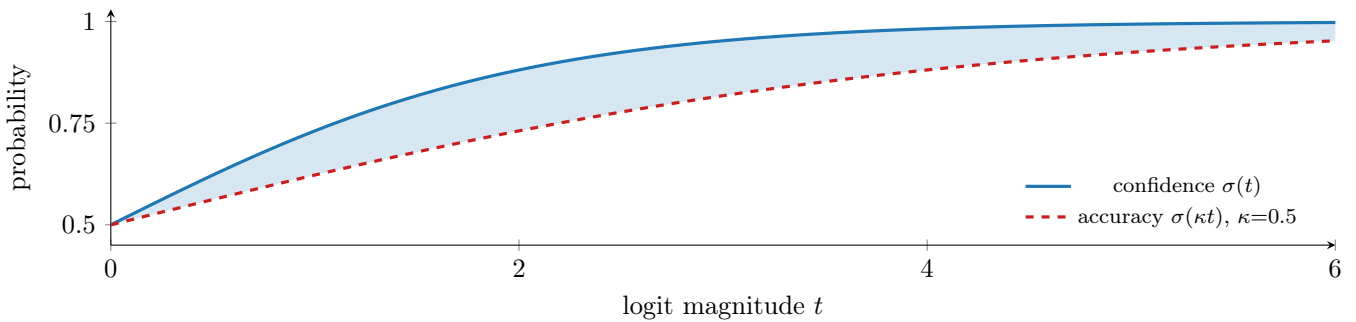

\section{Related work}\label{sec:related}

\subsection{Calibration of neural networks}
Modern networks trained with cross-entropy are systematically miscalibrated, typically
over-confident on natural images~\citep{guo2017calibration}. Temperature
scaling~\citep{guo2017calibration}, dividing the logits by a single scalar fit by
minimizing NLL on a labeled validation set, remains the most widely used post-hoc fix,
extended to class-wise and Dirichlet variants~\citep{kull2019dirichlet}. Expected
calibration error (ECE) and reliability diagrams~\citep{naeini2015ece} are the standard
diagnostic; large-scale studies~\citep{minderer2021revisiting} show calibration is
sensitive to architecture, augmentation, and training length even in-distribution, and
training-time remedies such as focal loss~\citep{mukhoti2020focal} trade accuracy for
calibration by reshaping the loss directly. Ovadia et al.~\citep{ovadia2019trust} showed that all of
these fixes degrade under distribution shift on image classifiers; the present paper
supplies, for GNNs, the mechanism that was left as an empirical observation there.

\subsection{In-distribution calibration of GNNs}
GNNs miscalibrate in ways that are structurally distinct from image classifiers: rather
than being uniformly over-confident, they are frequently under-confident in-distribution,
and the size of the gap correlates with graph properties such as node degree and
homophily~\citep{wang2021cagcn,hsu2022gats}. A first family of methods learns a per-node
temperature from an auxiliary model trained on labeled data: CaGCN~\citep{wang2021cagcn}
regresses temperatures with a second GCN; GATS~\citep{hsu2022gats} attends over
structural covariates identified as drivers of miscalibration; GETS~\citep{zhuang2025gets}
ensembles several temperature-generating experts in a mixture-of-experts head;
WATS~\citep{li2025wats} conditions the temperature on heat-kernel graph-wavelet features
that encode multi-scale structure without relying on neighbours' logits;
SimCalib~\citep{tang2024simcalib} ties calibration error directly to nodewise embedding
similarity. A second family regularizes the training loss itself rather than
post-processing logits, e.g.\ the graph calibration loss of Wang et al.~\citep{wang2022gcl}. All of
these methods are trained and validated on data from the \emph{same} distribution as the
test graph; none models, or is evaluated under, distribution shift, and, as I discuss
in Section~\ref{subsec:disc-sourcefit}, every one of them inherits the failure mode of
ordinary temperature scaling once the target distribution moves. A related line
quantifies predictive \emph{uncertainty} rather than calibration per se, e.g.\ JuryGCN's
jackknife variance estimates~\citep{kang2022jurygcn}; such estimates are informative but
do not by themselves specify a recalibration temperature.

\subsection{GNN reliability under distribution shift}
The gap this paper targets, calibration \emph{under shift}, has so far been approached
empirically rather than analytically. Trivedi et al.~\citep{trivedi2024gduq} document, across covariate,
concept, and graph-size shifts, that increased GNN expressivity does not by itself improve
confidence-indicator quality, and propose G-$\Delta$UQ, a stochastic-anchoring scheme that
estimates \emph{epistemic} uncertainty from a single model and uses it to modulate
confidence; this consistently improves out-of-distribution ECE relative to Monte Carlo
dropout and vanilla temperature scaling, but offers no closed-form account of \emph{why}
shift moves calibration or of what the optimal correction is, and it requires retraining
the base GNN with a stochastic-centering objective rather than recalibrating a frozen
model. Huang et al.~\citep{huang2023conformal} take a complementary, distribution-free route:
conformalized GNNs (CF-GNN) wrap any trained GNN with a topology-aware conformal
predictor that provably achieves a target marginal coverage, trading the pointwise notion
of calibration this paper studies for a set-valued guarantee that holds under
exchangeability but is not designed to track \emph{how} miscalibration grows as homophily
or feature noise change. My closed-form slope $\kap$ is, to my knowledge, the first
result that predicts the sign, magnitude, and optimal correction of GNN miscalibration
under shift analytically, and the first to explain \emph{why} a single global temperature
should be expected to suffice.

\subsection{Distribution shift and out-of-distribution generalization on graphs}
A separate and much larger literature targets \emph{accuracy}, rather than calibration,
under graph distribution shift. Invariance-based training regularizes representations to
be stable across environments (EERM~\citep{wu2022eerm}, SR-GNN~\citep{zhu2021srgnn}),
out-of-distribution generalization is benchmarked systematically by
GOOD~\citep{gui2022good}, and OOD-GNN~\citep{li2022oodgnn} proposes an architecture-level
generalization bound. You et al.~\citep{you2023gda} derive a model-based risk bound for graph domain
adaptation and use it to design a spectral regularizer (GDA-SpecReg), theory-grounded in
the same spirit as my approach, but aimed at bounding the target \emph{error}, not at
characterizing the resulting \emph{confidence-accuracy} gap, and it requires access to
the source graph at adaptation time. GTrans~\citep{jin2023gtrans} adapts the graph itself
(features and edges) at test time rather than the model, but again optimizes accuracy.
None of these methods produce, or are evaluated on, calibrated probabilities.

\subsection{Source-free and test-time adaptation}
Outside calibration, source-free test-time adaptation removes the requirement of access
to source data at deployment by adapting model parameters using only the target input
distribution: entropy minimization (Tent~\citep{wang2021tent}), self-supervised
test-time training~\citep{sun2020ttt} (specialized to graphs
by Wang et al.~\citep{wang2022ttgnn}), and source-hypothesis transfer (SHOT~\citep{liang2020shot})
all optimize an unsupervised proxy for \emph{accuracy} on the target data. STAC
(Section~\ref{sec:method}) is source-free in the same sense (only a scalar threshold is
carried from the source model), but optimizes a \emph{calibration} objective rather than
accuracy, and, following Corollary~\ref{cor:global}, deliberately restricts the
adaptation to a single global scalar under homogeneous shift rather than adapting model
weights.

\subsection{Label-free performance estimation under shift}
ATC~\citep{garg2022atc} and disagreement-based generalization-gap
estimators~\citep{jiang2022disagreement} predict a model's accuracy on unlabeled, shifted
data from confidence statistics and prediction disagreement under perturbation,
respectively. Theorem~\ref{thm:kappa} shows that these estimators supply exactly the
missing ingredient a calibration method needs (a single target-accuracy scalar pins down
$\Tstar=1/\kap$), and Section~\ref{sec:exp} shows this is simultaneously STAC's strength
(it reduces recalibration to one well-studied sub-problem) and its weakness (existing
estimators are not accurate enough under graph shift; Section~\ref{subsec:disc-crux}).

\subsection{Homophily-aware and physics-inspired GNN architectures}
My signal coefficient $c_K(h)$ recovers the homophily threshold that motivates a family
of spectral and physics-inspired GNN architectures designed to remain accurate across the
homophily spectrum: FAGCN~\citep{bo2021fagcn} learns a signed low/high-pass filter,
GPR-GNN~\citep{chien2021gprgnn} learns signed propagation weights, and
GraphCON~\citep{rusch2022graphcon} and A-DGN~\citep{gravina2023adgn} use
oscillatory/anti-symmetric ODE dynamics to propagate information without the
over-smoothing that erodes signal at low homophily. These works change the
\emph{architecture} to remain accurate across $h$; I hold the architecture fixed and
instead show how $h$ governs the confidence-accuracy gap of whatever architecture
produced the logits, via $c_K(h)$ in Proposition~\ref{prop:K}. The two views are
complementary: a homophily-robust architecture changes the operating point, but the
calibration slope, and the optimality of a single temperature, follow from
Theorem~\ref{thm:kappa} regardless of which architecture produced the logits, consistent
with my real-graph experiments, which calibrate SIGN/logistic-regression features, not
the linear aggregation of Section~\ref{sec:setup} directly.

\subsection{Summary and positioning}
Table~\ref{tab:related} summarizes the positioning of this paper against the eleven most
closely related methods along the axes that determine deployability under shift: whether
the method needs labels that are assumed representative of the test distribution,
whether it explicitly models distribution shift, the granularity of the correction (a
single global scalar vs.\ a per-node function), and whether it comes with a theoretical
guarantee. Only the present work combines a closed-form, shift-aware theory with a fully
label-free correction; the price, made explicit in Section~\ref{subsec:disc-crux}, is
that its practical reliability inherits the reliability of whichever label-free accuracy
estimator is plugged in.

\begin{table*}[t]
\centering\footnotesize
\caption{\textbf{Positioning against closely related work.} ``Test-distr.\ labels''
asks whether the method needs labeled data assumed to represent the distribution it is
applied to (all in-distribution calibrators do; STAC and the oracle baseline do not).
``Shift-aware'' asks whether the method's design or analysis explicitly targets a
\emph{change} of distribution. ``Granularity'' is the unit of correction. ``Theory''
records the type of formal guarantee, if any.}
\label{tab:related}
\setlength{\tabcolsep}{3pt}
\begin{tabular}{@{}>{\raggedright\arraybackslash}p{2.6cm}c>{\centering\arraybackslash}p{1.4cm}>{\centering\arraybackslash}p{1.5cm}>{\centering\arraybackslash}p{1.8cm}>{\raggedright\arraybackslash}p{2.75cm}@{}}
\toprule
Method & Year & Test-distr.\ labels & Shift-aware & Granularity & Theoretical guarantee\\
\midrule
Temperature scaling~\citep{guo2017calibration} & 2017 & Yes & No & Global & Asymptotic NLL-optimality\\
CaGCN~\citep{wang2021cagcn} & 2021 & Yes & No & Per-node & None\\
GATS~\citep{hsu2022gats} & 2022 & Yes & No & Per-node & None\\
CF-GNN~\citep{huang2023conformal} & 2023 & Yes & No & Sets (per-node) & Coverage (conformal)\\
GDA-SpecReg~\citep{you2023gda} & 2023 & No & Yes & Global (bound) & Risk bound (accuracy)\\
SimCalib~\citep{tang2024simcalib} & 2024 & Yes & No & Per-node & None\\
G-$\Delta$UQ~\citep{trivedi2024gduq} & 2024 & Yes (retrain) & Partial (empir.) & Per-node & None\\
GETS~\citep{zhuang2025gets} & 2025 & Yes & No & Per-node & None\\
WATS~\citep{li2025wats} & 2025 & Yes & No & Per-node & None\\
\textbf{STAC (this work)} & n/a & \textbf{No} & \textbf{Yes} & \textbf{Global} & \textbf{Closed form ($\kap$, ECE bound)}\\
\bottomrule
\end{tabular}
\end{table*}

\section{Problem setup}\label{sec:setup}

\subsection{Problem statement}
I am given a source graph $G_s=(A_s,X_s,Y_s)$ with node labels, on which a GNN $f$ is
trained and calibrated (e.g.\ by ordinary temperature scaling on a source validation
split, so $f$ is calibrated on $G_s$). At test time I observe a target graph
$G_t=(A_t,X_t)$ from the \emph{same} feature/label space but a \emph{different}
distribution (different homophily, different feature noise, or both), with \emph{no}
target labels and no further access to $G_s$ or $Y_s$. The model $f$ is frozen:
re-training is not assumed to be an option in the deployment settings I target
(streaming graphs, on-device inference, or simply the cost of retraining a production
model). The question I answer is: how does the confidence-accuracy relationship of
$f$'s frozen predictions change from $G_s$ to $G_t$, and how can it be corrected without
labels?

\subsection{Contextual stochastic block model}
\textbf{Nodes and labels.} $N$ nodes are partitioned into $K$ balanced classes with
labels $y_i$ i.i.d.\ uniform on $\{1,\dots,K\}$.
\textbf{Features.} $x_i=\mu_{y_i}+\xi_i$ with $\xi_i\sim\mathcal N(0,\sigma^2 I)$ i.i.d.,
class means $\|\mu_c\|=r$, $\sum_c\mu_c=0$ (binary case: $\mu_\pm=\pm r\hat\mu$). The
scale $r$ relative to the noise level $\sigma$ is the \textbf{signal-to-noise ratio}
$\rho=r^2/\sigma^2$: large $\rho$ means classes are easy to separate from features alone.
\textbf{Edges and homophily.} Edges are intra-class with probability $p$ and inter-class
with probability $q$; the \textbf{edge homophily} $h=p/(p+q(K{-}1))$ is the probability
that a uniformly random edge joins two same-class nodes: $h>1/K$ is homophilic,
$h<1/K$ heterophilic, and $h=1/K$ is the point at which class structure is invisible to
a same-class-counting statistic. This is the standard contextual SBM used to study
information-plus-structure learning problems~\citep{deshpande2018csbm}.

\subsection{Linear GNN and calibration}
I analyze mean aggregation $z_i=\frac1{d_i}\sum_{j\sim i}x_j$ and the symmetric-normalized
GCN operator $\hat A=\tilde D^{-1/2}(A{+}I)\tilde D^{-1/2}$~\citep{kipf2017gcn}, followed
by a linear classifier calibrated on the source graph; multi-hop
SIGN~\citep{frasca2020sign,wu2019sgc} features are used on real data in
Section~\ref{sec:exp}. With logit $\delta_i$, confidence is $\soft(|\delta_i|)$ and the
prediction is correct iff $\operatorname{sign}(\delta_i)=y_i$; ECE bins predictions by
confidence and averages $|\text{acc}-\text{conf}|$~\citep{naeini2015ece,guo2017calibration}.
Temperature scaling maps $\delta\mapsto\delta/T$ and never changes predictions, so it
cannot help or hurt accuracy, only calibration.

\subsection{Why a linear GNN on a CSBM?}\label{subsec:why-linear}
A skeptical reader may ask why I restrict the closed-form analysis to a linear classifier
on Gaussian CSBM features rather than a deep, nonlinear GNN. I follow a long tradition in
learning theory (of which the CSBM itself is a canonical graph
instance~\citep{deshpande2018csbm}) of trading architectural realism for a model simple
enough to solve exactly, then checking numerically how far its predictions travel toward
the regime of interest. Three considerations support this choice here. First, mean/GCN
aggregation followed by a linear read-out is not merely a toy: it is the exact computation
performed by SGC~\citep{wu2019sgc} and, feature-wise, by SIGN~\citep{frasca2020sign}, both
competitive with deeper GNNs on the heterophilous benchmarks I use in
Section~\ref{sec:exp}, so the model class I analyze is a real, widely used GNN family,
not merely an approximation of one. Second, the object I characterize, the map from
logit magnitude to accuracy, depends on the \emph{aggregation step}, which mixes same-
and different-class neighbours in a homophily-dependent ratio, far more than on the
specific nonlinearity of the classifier head; this is why the closed-form slope, derived
for a linear head, tracks a self-loop GCN's simulated slope closely
(Section~\ref{subsec:exp-theory}, Fig.~\ref{fig:ext}) and a logistic-regression head on
SIGN features on five real graphs (Section~\ref{subsec:exp-real}). Third, and most
importantly, every qualitative claim the theory makes (the sign of over/under-confidence,
the existence of a single optimal temperature, the non-benefit of per-node calibration
under homogeneous shift) is falsifiable and is checked against simulation and real data
throughout Section~\ref{sec:exp}; I treat the tractable model as a source of falsifiable,
quantitative predictions about a real computation (mean/GCN aggregation plus a linear or
near-linear head), not as a claim about mechanism inside arbitrarily deep, nonlinear GNNs.

\section{Theoretical analysis}\label{sec:theory}
All proofs are in Appendix~\ref{app:proofs}; results hold in a large-degree Gaussian
regime and are confirmed numerically (Section~\ref{sec:exp}). Figure~\ref{fig:mega}
previews how the three parts of the paper (theory, method, and validation) fit
together.

\begin{figure*}[t]
\centering
\begin{tikzpicture}[
  font=\small,
  b/.style={rounded corners,draw,align=center,inner sep=3pt,minimum height=1.0cm,text width=3.15cm,font=\small},
  th/.style={b,fill=cSrc!12,draw=cSrc!60},
  me/.style={b,fill=cTgt!13,draw=cTgt!72},
  va/.style={b,fill=cMod!14,draw=cMod!65},
  a/.style={-{Latex[length=2mm]},thick,gray!55!black},
  ba/.style={-{Latex[length=3mm]},very thick,black!45}]
\node[th](t1){CSBM $+$ aggregation:\\ signal $(2h{-}1)$, noise $\propto\!1/d$};
\node[th,right=5mm of t1](t2){calibration slope\\ $\kap(h_s,h_t,\rho)$ \ (Eq.~\ref{eq:kappa})};
\node[th,right=5mm of t2](t3){conf $\soft(t)$ vs.\ acc $\soft(\kap t)$\\ $\Rightarrow\ \Tstar{=}1/\kap$};
\draw[a](t1)--(t2); \draw[a](t2)--(t3);
\node[me,below=8mm of t1](m1){source-calibrate\\ \& \emph{freeze} model};
\node[me,right=5mm of m1](m2){target (shifted, unlabeled):\\ perturb $\to$ estimate $\hat a$};
\node[me,right=5mm of m2](m3){STAC sets a \emph{single}\\ $\Tstar{=}1/\hat\kap$};
\draw[a](m1)--(m2); \draw[a](m2)--(m3);
\node[va,below=8mm of m1](v1){synthetic: $\kap$ \& $\Tstar$\\ exact ($r{=}0.99$)};
\node[va,right=5mm of v1](v2){5 real graphs: one oracle\\ $T$ recovers ECE$\,\le\!0.024$};
\node[va,right=5mm of v2](v3){open problem:\\ label-free accuracy est.};
\draw[a](v1)--(v2); \draw[a](v2)--(v3);
\draw[ba](t2.south)--(m2.north) node[midway,right=1pt,font=\footnotesize\bfseries,text=cTgt!62!black]{apply};
\draw[ba](m2.south)--(v2.north) node[midway,right=1pt,font=\footnotesize\bfseries,text=cMod!55!black]{test};
\node[left=2mm of t1,rotate=90,anchor=south,font=\bfseries\footnotesize,text=cSrc!62!black]{THEORY};
\node[left=2mm of m1,rotate=90,anchor=south,font=\bfseries\footnotesize,text=cTgt!68!black]{METHOD};
\node[left=2mm of v1,rotate=90,anchor=south,font=\bfseries\footnotesize,text=cMod!60!black]{VALIDATION};
\end{tikzpicture}
\caption{\textbf{Overview.} \emph{Theory} (top): on a CSBM, aggregation scales the class
signal by $(2h{-}1)$ and the noise by $1/d$, giving a closed-form calibration slope
$\kap$; the model's confidence is $\soft(t)$ but its accuracy is $\soft(\kap t)$, so the
optimal temperature is $\Tstar{=}1/\kap$. \emph{Method} (middle): freeze the
source-calibrated model and, on the shifted unlabeled target, estimate accuracy from
perturbations to set one temperature (STAC, Section~\ref{sec:method}).
\emph{Validation} (bottom): the closed form is exact in simulation and one oracle
temperature recovers calibration on five real graphs; the only bottleneck is label-free
accuracy estimation (Section~\ref{subsec:disc-crux}).}
\label{fig:mega}
\end{figure*}
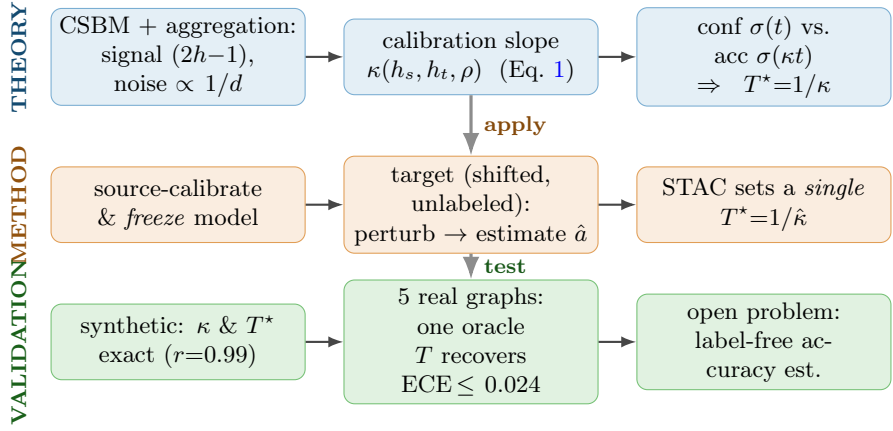

\subsection{The aggregated discriminant coordinate}
\begin{lemma}[Aggregated coordinate]\label{lem:agg}
Under mean aggregation, conditional on $y_i=+1$ and degree $d_i$, the discriminant
coordinate $g_i=\hat\mu^\top z_i$ has mean $(2h{-}1)r$ and variance
$\big(4h(1{-}h)r^2+\sigma^2\big)/d_i$, and is asymptotically Gaussian as $d_i\to\infty$.
\end{lemma}

\begin{remark}[Reading the lemma]\label{rem:agg}
The mean $(2h{-}1)r$ is exactly the excess of same-class over different-class neighbours,
scaled by the class-mean separation $r$; it is positive for homophilic graphs ($h>1/2$),
zero at $h=1/2$ where a neighbour is a coin flip over the two classes, and negative for
heterophilic graphs, where aggregation systematically pulls a node's representation
\emph{toward the wrong class}. The variance combines two independent noise sources, label
randomness among neighbours ($4h(1{-}h)r^2$) and feature noise ($\sigma^2$), and both
shrink as $1/d_i$: high-degree nodes average out noise exactly as one would expect from a
sample mean of $d_i$ roughly independent terms. Gaussianity is not an approximation
tacked on for convenience; it is a genuine Lindeberg central-limit effect over the $d_i$
neighbours, which is why the closed forms below are accurate already at moderate degree
($d\approx20$, Section~\ref{sec:exp}).
\end{remark}

\subsection{The calibration slope}
\begin{theorem}[Calibration slope]\label{thm:kappa}
Let the logit be $\delta=a\,g$, so $\delta\mid y{=}{+}1\sim\mathcal N(m,s^2)$ with
$m=a(2h{-}1)r$, $s^2=a^2(4h(1{-}h)r^2+\sigma^2)/d$, and $\delta\mid y{=}{-}1\sim\mathcal
N(-m,s^2)$. With $\kap:=2m/s^2$, for every $t$
\[
\Pr(\text{correct}\mid\delta{=}t)=\soft(\kap t),\qquad \text{confidence}(t)=\soft(t).
\]
Hence the model is over-confident iff $\kap<1$, under-confident iff $\kap>1$, calibrated
iff $\kap=1$; the NLL-optimal temperature is $\Tstar=1/\kap$; and with $a$ fixed by source
calibration ($\kap(h_s){=}1$),
\begin{equation}\label{eq:kappa}
\kap(h_t)=\frac{(2h_t{-}1)\big(1+4h_s(1{-}h_s)\rho\big)}
{(2h_s{-}1)\big(1+4h_t(1{-}h_t)\rho\big)},
\end{equation}
independent of degree.
\end{theorem}

\begin{remark}[Reading Eq.~\eqref{eq:kappa}]\label{rem:kappa}
Equation~\eqref{eq:kappa} says homophily controls the \emph{sign and strength} of the
aggregation signal, $(2h{-}1)$: moving $h_t$ toward $\tfrac12$ shrinks the signal faster
than the variance, so $\kap<1$ and the frozen model becomes over-confident; a target more
homophilous than the source ($h_t>h_s$, both $>1/2$) gives $\kap>1$ and the model becomes
under-confident, because the same logit magnitude that was appropriate for a noisier
source aggregation now under-states how reliable the cleaner target signal actually is.
The formula is antisymmetric in a useful sense: $\kap(h_s)=1$ by construction, and $\kap$
crosses zero exactly when $h_t=1/2$, where the aggregated signal vanishes and no logit
magnitude, however large, can be predictive, an early warning that the bound in
Proposition~\ref{prop:bound} becomes uninformative there. The dependence on $\rho$ enters
only through the combination $4h(1{-}h)\rho$, i.e.\ SNR matters only insofar as it
competes with the same same-/different-class mixing variance that homophily itself
controls; at $\rho\to\infty$ (noiseless features) Eq.~\eqref{eq:kappa} reduces to the
purely structural ratio $(2h_t{-}1)/(2h_s{-}1)$.
\end{remark}

\begin{corollary}[Global temperature is optimal]\label{cor:global}
Under homogeneous shift (a single $\kap$), $\Tstar=1/\kap$ makes the model equal to the
true conditional, attaining minimum population NLL and zero population ECE. Any per-node
temperatures $\{T_i\}$ with $T_i\ne\Tstar$ on a positive-measure set strictly increase NLL.
Per-node recalibration can therefore only help under \emph{heterogeneous} shift.
\end{corollary}

\begin{remark}[Why this corollary matters operationally]\label{rem:global}
Corollary~\ref{cor:global} is a negative result stated as a positive one: it says that
\emph{any} extra modeling effort spent on a per-node temperature function is provably
wasted under homogeneous shift, no matter how expressive the per-node model is, because
the population-optimal per-node solution is already the constant function $\Tstar$. This
is a strong, falsifiable prediction about methods like CaGCN and GATS
(Section~\ref{sec:related}), which fit a per-node temperature \emph{even when} the shift
they are correcting for is homogeneous; Section~\ref{subsec:exp-method} tests the
prediction directly with a dedicated ablation (STAC-global vs.\ STAC-full) and finds it
holds: the per-node stage changes error-detection ranking only marginally, and not more
than an ablation that discards graph structure entirely.
\end{remark}

\subsection{Extensions}
\begin{proposition}[Self-loop GCN]\label{prop:sl}
For $\hat A=\tilde D^{-1/2}(A{+}I)\tilde D^{-1/2}$, with $A(h)=1+d(2h{-}1)$ and
$B(h)=4d\,h(1{-}h)\rho+(d{+}1)$,
$\kap_{SL}(h_t)=A(h_t)B(h_s)/[A(h_s)B(h_t)]$, which reduces to \eqref{eq:kappa} as
$d\to\infty$. The self term makes the signal survive down to $h=(1{-}1/d)/2<\tfrac12$,
formalizing why self-loops help under heterophily.
\end{proposition}

\begin{remark}
The self-loop term adds a deterministic same-class contribution ($+r$ with certainty,
since a node is always its own class) that does not exist for a node's noisy neighbours.
This shifts the zero-crossing of the aggregated signal from exactly $h=1/2$ (mean
aggregation) to $h=(1{-}1/d)/2$, slightly below $1/2$: self-loops buy a small but
degree-dependent margin of heterophily tolerance before the signal vanishes, consistent
with the common empirical observation that adding self-loops helps GCNs on mildly
heterophilic graphs.
\end{remark}

\begin{proposition}[$K$ classes]\label{prop:K}
With symmetric simplex means, conditional on $y_i=c$, $\E[z_i]=c_K(h)\,\mu_c$ with
$c_K(h)=(hK{-}1)/(K{-}1)$ (zero at $h=1/K$, equal to $2h-1$ for $K{=}2$). Homophily shift
rescales the whole logit vector by $c_K(h_t)/c_K(h_s)$, so a single temperature suffices.
\end{proposition}

\begin{remark}
$c_K(h)$ is the natural $K$-class generalization of $(2h{-}1)$: a same-class neighbour
(probability $h$) contributes $+\mu_c$, and a different-class neighbour (probability
$1{-}h$, uniform over the remaining $K{-}1$ classes) contributes $-\mu_c/(K{-}1)$ in
expectation because the simplex means sum to zero. The zero-crossing moves from $1/2$ to
$1/K$ as the number of classes grows: the more classes there are, the less homophily is
needed before a random neighbour is more likely to be a different class than not
(Fig.~\ref{fig:schem} in Section~\ref{sec:exp} illustrates this schematically). Because
the rescaling by $c_K(h_t)/c_K(h_s)$ is common to every class direction, it acts on the
full $K$-dimensional logit vector exactly as a scalar temperature would, which is why a
single global temperature remains sufficient even in the multi-class case, as confirmed
numerically for $K=3,4,5$ in Section~\ref{subsec:exp-theory}.
\end{remark}

\begin{proposition}[Covariate shift]\label{prop:cov}
At fixed $h$, if the target feature-noise variance changes $\sigma^2\!\mapsto\!\sigma^2+\tau^2$
($\gamma=\tau^2/\sigma^2$), then with source calibration
$\kap_{\mathrm{cov}}=\frac{4h(1{-}h)\rho+1}{4h(1{-}h)\rho+1+\gamma}<1$ and
$\Tstar=1+\gamma/(4h(1{-}h)\rho+1)>1$: covariate shift always induces over-confidence,
monotonically in $\gamma$.
\end{proposition}

\begin{remark}
Unlike homophily shift, covariate shift is one-directional: because extra feature noise
can only inflate the logit variance $s^2$ (it cannot change the signal $m$, which depends
only on the class means and homophily), $\kap_{\mathrm{cov}}$ can only fall below $1$.
This is a clean, falsifiable prediction (covariate shift never causes
under-confidence), which Section~\ref{subsec:exp-real-cov} confirms on five real graphs:
the confidence-accuracy gap grows monotonically, and only in the over-confident direction,
as the relative feature-noise strength increases.
\end{remark}

\begin{proposition}[ECE bound]\label{prop:bound}
$\mathrm{ECE}\le\tfrac14|\kap-1|\,\E|\delta|$; it vanishes as $\kap\to1$ and grows with
the logit magnitude.
\end{proposition}

\begin{remark}
The bound follows from the $\tfrac14$-Lipschitz constant of the sigmoid and turns
Theorem~\ref{thm:kappa}'s \emph{pointwise} statement into a single scalar diagnostic:
because $\E|\delta|$ is observable from the model's own logits on the target graph
\emph{without labels}, this bound gives a fully label-free, if conservative, upper
estimate of how miscalibrated a frozen model can possibly be under a shift of a given
severity $|\kap-1|$, useful as an early-warning signal even before attempting the
recalibration of Section~\ref{sec:method}. The bound is tightest, and the recalibration
correspondingly most consequential, exactly when confident predictions ($|\delta|$ large)
coincide with a shift severe enough to move $\kap$ far from $1$; Section~\ref{sec:exp}
confirms empirically that the bound is tight near the source and looser, but never
violated, far from it.
\end{remark}

\section{STAC: a source-free, label-free recalibration method}\label{sec:method}
Theorem~\ref{thm:kappa} says the entire job is to estimate one number, $\Tstar=1/\kap$,
from \emph{unlabeled} target data (Fig.~\ref{fig:pipeline}, Algorithm~\ref{alg:stac}). I
call the resulting method \textbf{STAC} (Structure-aware Test-time Adaptive
Calibration). STAC \emph{(i)} estimates target accuracy $\hat a$ from unlabeled data by
combining thresholded confidence~\citep{garg2022atc} with prediction disagreement under
input/edge perturbations~\citep{jiang2022disagreement}, then \emph{(ii)} chooses the
single $T$ so that mean confidence equals $\hat a$. The method is source-free (only a
scalar threshold is carried from training), adds $O(1)$ parameters, and needs no GPU.

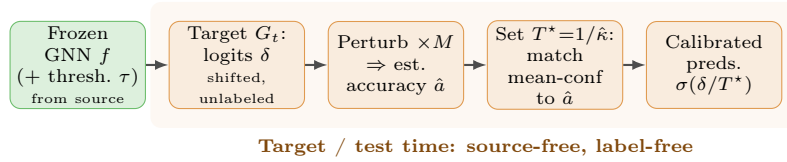
\begin{figure*}[t]
\centering
\begin{tikzpicture}[
  node distance=4mm and 3mm,
  box/.style={rounded corners,draw,align=center,minimum height=1.15cm,
              inner sep=2pt,font=\scriptsize,text width=1.65cm},
  mod/.style={box,fill=cMod!16,draw=cMod!70},
  tgt/.style={box,fill=cTgt!14,draw=cTgt!75},
  ar/.style={-{Latex[length=1.8mm]},thick,gray!60!black}]
\node[mod](f){Frozen GNN $f$\\ ($+$ thresh.\ $\tau$)\\ \tiny from source};
\node[tgt,right=of f](lg){Target $G_t$:\\ logits $\delta$\\ \tiny shifted, unlabeled};
\node[tgt,right=of lg](est){Perturb $\times M$\\ $\Rightarrow$ est.\\ accuracy $\hat a$};
\node[tgt,right=of est](T){Set $\Tstar{=}1/\hat\kap$:\\ match mean-conf\\ to $\hat a$};
\node[tgt,right=of T](out){Calibrated\\ preds.\\ $\soft(\delta/\Tstar)$};
\draw[ar](f)--(lg); \draw[ar](lg)--(est); \draw[ar](est)--(T); \draw[ar](T)--(out);
\begin{scope}[on background layer]
  \node[fill=cTgt!5,rounded corners,fit=(lg)(est)(T)(out),inner sep=2.4mm,
        label={[cTgt!60!black,font=\scriptsize\bfseries]below:Target / test time: source-free, label-free}]{};
\end{scope}
\end{tikzpicture}
\caption{\textbf{STAC.} At test time the frozen source model is applied to the shifted,
unlabeled target graph; a few perturbed passes yield a label-free accuracy estimate
$\hat a$, which fixes a single temperature $\Tstar=1/\hat\kap$ so that mean confidence
matches $\hat a$. No target labels and no source data are used.}
\label{fig:pipeline}
\end{figure*}

\begin{algorithm}[t]
\small
\caption{STAC: label-free test-time recalibration (global stage)}\label{alg:stac}
\begin{algorithmic}[1]
\Require frozen model $f$, source threshold $\tau$, target graph $G_t$, passes $M$
\State $\delta\gets$ logits of $f$ on $G_t$;\quad $\hat a_{\text{ATC}}\gets$ frac.\ of
       $\max\soft(\delta)\!\ge\!\tau$
\State run $M$ perturbed passes; $\hat a_{\text{GDE}}\gets 1-$ mean pairwise disagreement
\State $\hat a\gets\min(\hat a_{\text{ATC}},\hat a_{\text{GDE}})$
\State $\Tstar\gets\arg\min_T\big(\text{mean}_i\,\max_k\soft(\delta_i/T)_k-\hat a\big)^2$
\State \Return calibrated probabilities $\soft(\delta/\Tstar)$
\end{algorithmic}
\end{algorithm}

\paragraph{Complexity.} Algorithm~\ref{alg:stac} requires $M$ additional forward passes
of the (frozen) model over the target graph (no gradients, no retraining), plus an
$O(M^2)$ pairwise-disagreement computation per node and a one-dimensional line search for
$T$. In my experiments $M=6$--$10$ suffices, and the whole procedure runs in seconds on
CPU for graphs with tens of thousands of nodes; the dominant cost is the $M$ forward
passes, which is the same order of cost as computing $M$-sample Monte Carlo dropout
uncertainty, but here applied at test time to a deterministic frozen model via input and
edge perturbations rather than at training time to a stochastic one.

\paragraph{Optional per-node stage.} A second, optional stage modulates the global
temperature per node, $T_i=T_{\text{glob}}\exp(c\tanh(\phi_i^\top w+b))$, where $\phi_i$
collects structure-aware signals (feature- and edge-perturbation instability,
homophily-corrected neighbour-agreement, prediction entropy, margin, and log-degree), fit
by minimizing the same label-free binned calibration loss with a small ridge penalty on
$w$. By Corollary~\ref{cor:global}, this stage is \emph{not expected to help} under
homogeneous shift, and I include it purely as an empirical test of that prediction: an
ablation without structural signals (logits only) isolates whether any residual benefit
comes from graph structure specifically or merely from having a higher-capacity per-node
function. Section~\ref{subsec:exp-method} reports the outcome of this ablation.

\paragraph{Relation to the theory.} The reduction to a single scalar $\hat a$ is not a
simplifying choice I make for convenience; it is what Theorem~\ref{thm:kappa} and
Corollary~\ref{cor:global} \emph{prove} is sufficient under homogeneous shift. This is
what distinguishes STAC's global stage from the per-node calibrators surveyed in
Section~\ref{sec:related}: those methods use a per-node function because they have no
theoretical reason to expect a global scalar to be optimal, whereas STAC uses a global
scalar \emph{because} the theory says it is optimal, and only reaches for a per-node
correction as a deliberately falsifiable ablation.

\section{Experiments}\label{sec:exp}
All experiments run on CPU using only NumPy/SciPy/scikit-learn; code and data are
released. I report ECE (15 bins) throughout.

\subsection{Theory validation (synthetic)}\label{subsec:exp-theory}
On a binary CSBM ($h_s{=}0.8,\rho{=}1$), the measured slope $\kap=1/T_{\text{oracle}}$
matches the closed form \eqref{eq:kappa} across $h_t\in[0.55,0.9]$, capturing both
over- and under-confidence, with $\kap\approx1$ at the source (Fig.~\ref{fig:val}, left);
ECE grows from $0.004$ to $0.25$ as $h_t$ leaves $h_s$ and $\Tstar$ returns it to
$\le0.02$ (right). Figure~\ref{fig:schem} illustrates schematically why homophily has
this effect: aggregation averages same- and different-class neighbours, and the resulting
signal along the true class direction is scaled by $c_K(h)$, which vanishes and flips
sign at $h=1/K$ (Proposition~\ref{prop:K}). The self-loop slope $\kap_{SL}$ matches the
GCN simulation and is closer to it than the mean-aggregation formula in the heterophilous
regime (Fig.~\ref{fig:ext}, left); the signal coefficient $c_K(h)$ matches the measured
value \emph{exactly} for $K=3,4,5$ (right). The bound of Proposition~\ref{prop:bound}
holds throughout (Fig.~\ref{fig:bound}) and, as anticipated in the remark following it, is
tightest near the source homophily and looser, but never violated, far from it.

\begin{figure}[t]
\centering
\begin{tikzpicture}[every node/.style={circle,draw,minimum size=5.5mm,inner sep=0}]
\node[fill=cSrc!55](c) at (0,0){};
\foreach \a/\col in {20/cSrc!55,70/cAcc!55,120/cSrc!55,175/cMod!60,240/cAcc!55,300/cSrc!55}
  {\node[fill=\col](n\a) at (\a:1.35){}; \draw[gray!55] (c)--(n\a);}
\node[draw=none,anchor=west,font=\footnotesize] at (2.0,0.35)
  {same-class (\textcolor{cSrc}{$\bullet$}) fraction $h$};
\node[draw=none,anchor=west,font=\footnotesize] at (2.0,-0.35)
  {signal $c_K(h)=\tfrac{hK-1}{K-1}$, flips at $h{=}\tfrac1K$};
\end{tikzpicture}
\caption{\textbf{Why homophily controls calibration.} Aggregation averages same- and
different-class neighbours; the resulting signal along the true class is scaled by
$c_K(h)$, which vanishes and flips sign at $h=1/K$.}
\label{fig:schem}
\end{figure}
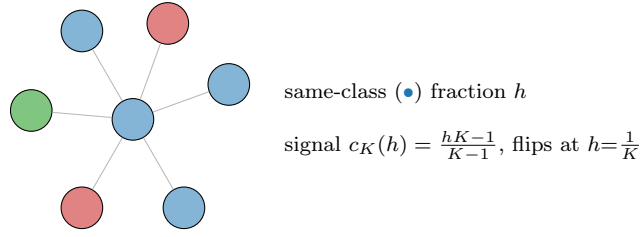

\begin{figure}[t]
\centering
\includegraphics[width=\linewidth]{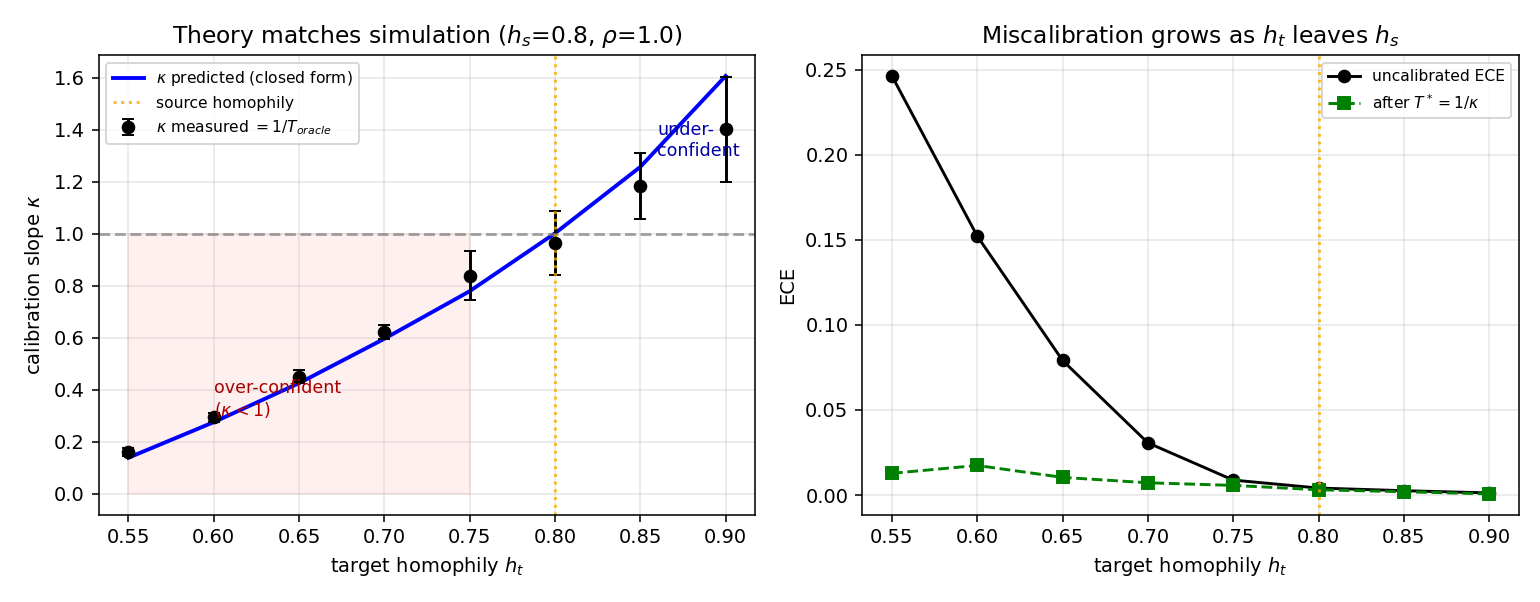}
\caption{\textbf{Theorem~\ref{thm:kappa}.} Closed-form slope vs.\ measured
$1/T_{\text{oracle}}$ (left); ECE before/after $\Tstar$ (right).}
\label{fig:val}
\end{figure}

\begin{figure}[t]
\centering
\includegraphics[width=\linewidth]{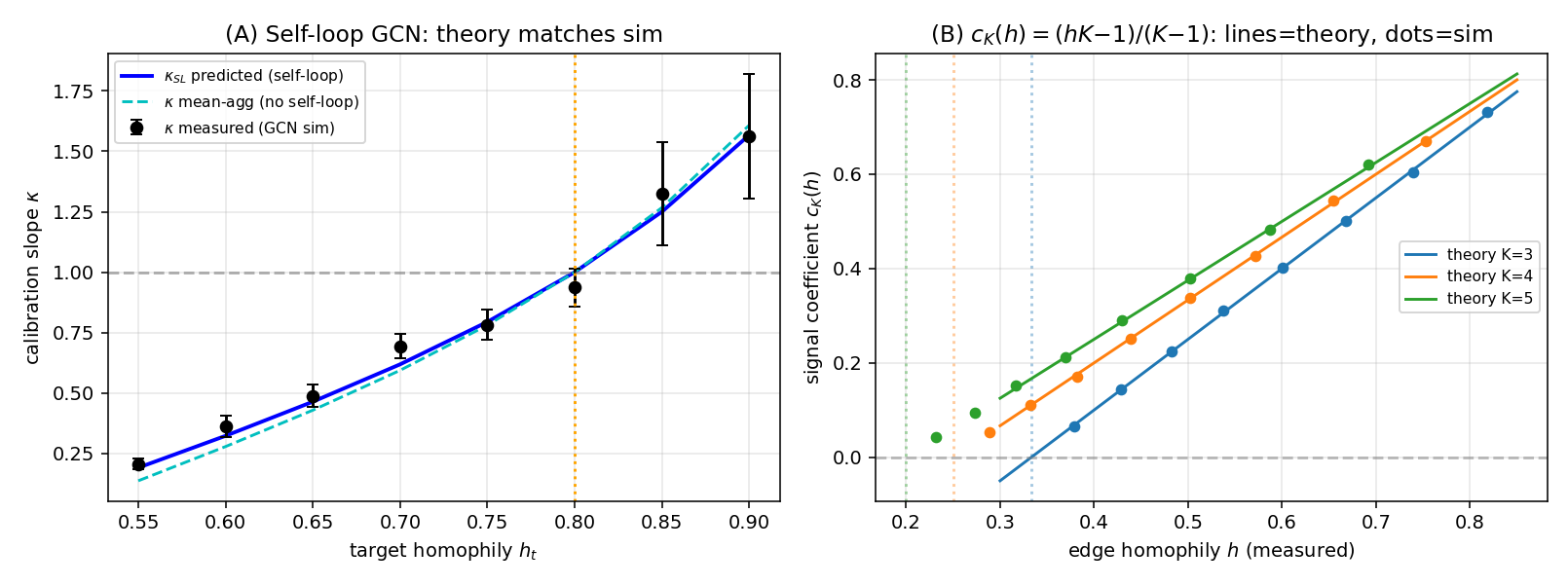}
\caption{\textbf{Extensions.} Self-loop GCN slope $\kap_{SL}$ (left) and the $K$-class
signal coefficient $c_K(h)$ (right): lines are theory, markers are simulation.}
\label{fig:ext}
\end{figure}

\begin{figure}[t]
\centering
\includegraphics[width=0.86\linewidth]{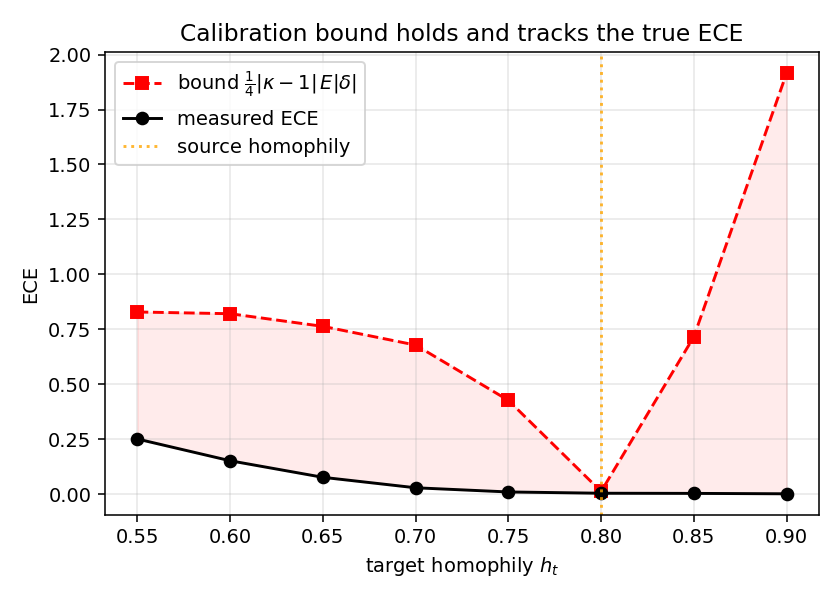}
\caption{\textbf{Proposition~\ref{prop:bound}.} The bound $\tfrac14|\kap-1|\,\E|\delta|$
upper-bounds the measured ECE and is tight near the source.}
\label{fig:bound}
\end{figure}

\subsection{The closed form predicts the \emph{optimal} temperature}
Beyond signs, the theory predicts the \emph{value} $\Tstar=1/\kap$. Over a grid of
source/target homophilies and SNRs on the binary CSBM, the predicted $\Tstar$ matches the
oracle temperature measured on labeled target data with Pearson $r=0.99$ and a mean
absolute error of $0.11$ temperature units (Fig.~\ref{fig:temp}). The closed form is thus
quantitatively predictive, not merely directional.

\begin{figure}[t]
\centering
\includegraphics[width=0.82\linewidth]{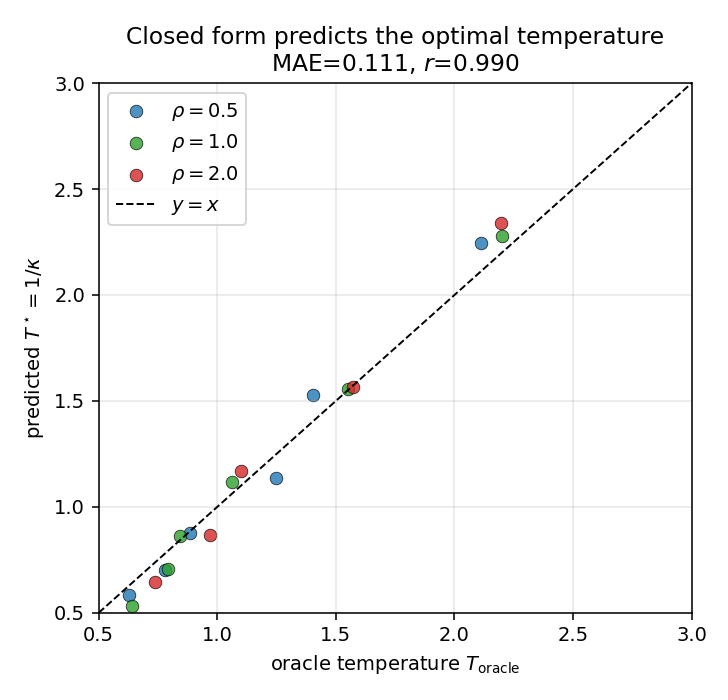}
\caption{\textbf{Optimal temperature.} Predicted $\Tstar=1/\kap$ vs.\ the oracle
temperature across $(h_s,h_t,\rho)$ configurations; points hug $y{=}x$
($r{=}0.99$, MAE $0.11$).}
\label{fig:temp}
\end{figure}

\subsection{STAC under shift (synthetic)}\label{subsec:exp-method}
On CSBM node classification with a trained 2-layer GCN, across feature, homophily, and
heterogeneous shifts, STAC's global stage roughly \emph{halves} the ECE gap versus
source-temperature scaling and approaches the oracle. Figure~\ref{fig:mech} makes the
mechanism concrete under increasing covariate-shift intensity. The left panel shows ECE
for four methods: the uncalibrated model degrades steadily as shift grows; source-TS,
fit once on the source validation split, barely helps and at high shift intensity is
close to indistinguishable from no calibration at all, because it targets the
\emph{source} accuracy rather than the falling target accuracy; STAC-global tracks
roughly half-way between source-TS and the oracle at every shift level; the oracle
(which uses target labels solely to fit $T$, never to retrain) stays flat and low
throughout. The right panel diagnoses \emph{why}: it plots the true target accuracy
against the two candidate proxies a calibrator could use. The `Source-TS assumption'
(dotted) is exactly the source validation accuracy, which is shift-invariant by
construction and therefore departs from the true target accuracy (solid) as shift
grows; this gap is the empirical face of $\kap$ moving away from $1$ as shift intensity
increases, and it is why source-fitted calibrators fail under shift
(Section~\ref{subsec:disc-sourcefit}). STAC's label-free accuracy estimate (dashed)
tracks the true target accuracy far more closely than the source assumption does, but not
perfectly; the residual gap between the dashed and solid curves is exactly what limits
STAC-global's ECE relative to the oracle in the left panel, and it is the same gap that
becomes unmanageable on real, heterophilous graphs in Section~\ref{subsec:exp-real}.

\begin{figure}[t]
\centering
\includegraphics[width=\linewidth]{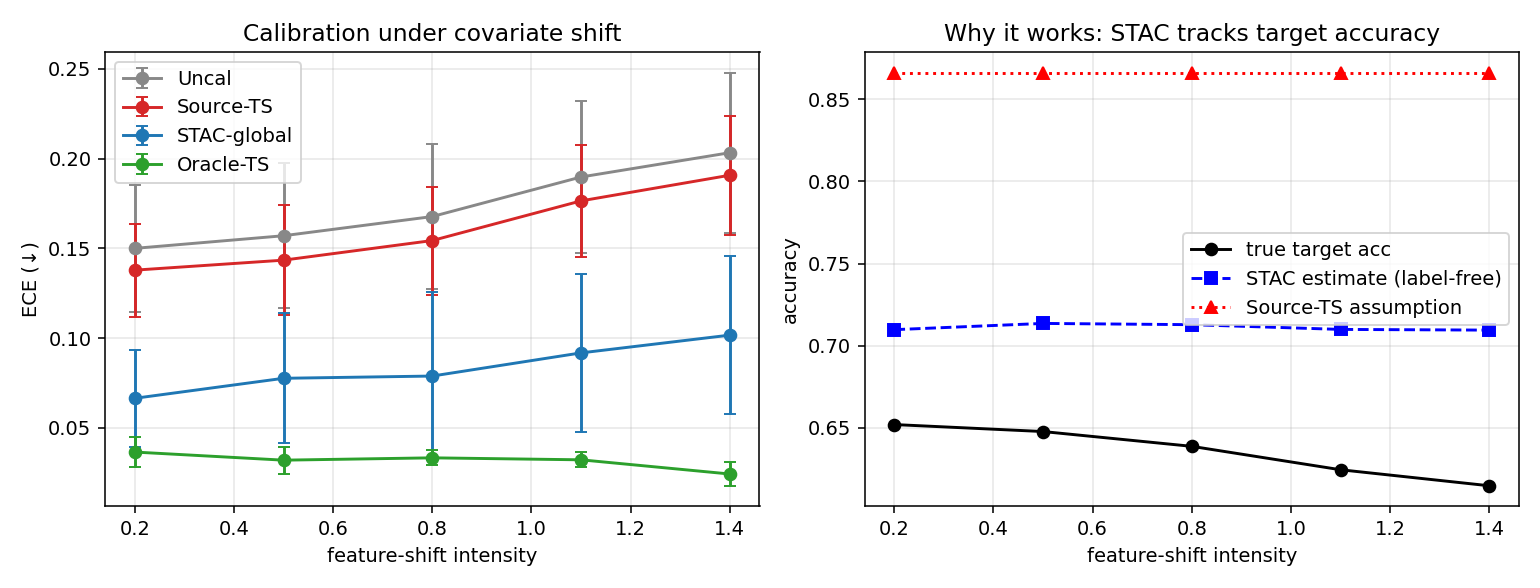}
\caption{\textbf{Mechanism (synthetic feature shift).} \emph{Left:} ECE vs.\ shift
intensity for the uncalibrated model, source-TS, STAC-global, and the oracle;
STAC-global consistently halves the source-TS gap. \emph{Right:} the true target
accuracy (solid) falls as shift grows, while the source-TS assumption (dotted, frozen at
the source validation accuracy) does not move at all: this divergence is the mechanism
behind source-TS's failure; STAC's label-free estimate (dashed) tracks the true accuracy
much more closely, though not exactly.}
\label{fig:mech}
\end{figure}

Under the heterogeneous (`mixed') shift, where only a random half of target nodes
receive a strong feature perturbation, so no single $\kap$ describes the whole graph, a
single temperature (STAC-global, identical to source-TS by construction once its level is
fixed) cannot, by definition, change the \emph{ranking} of nodes by confidence, so its
error-detection AUROC exactly matches the uncalibrated model's. Only the optional
per-node stage (STAC-full, Section~\ref{sec:method}), which modulates $T_i$ using
structure-aware signals, can move this ranking; in my runs it does so only marginally,
and an ablation without structural signals (logits only) performs comparably, indicating
that, at least for the perturbation signals I tried, graph structure does not add
error-detection information beyond what the model's own logits already carry. This is
precisely the falsifiable prediction of Corollary~\ref{cor:global} put to the test: a
single global scale is a genuinely different problem from finding a useful per-node
ranking signal, and my attempt at the latter with off-the-shelf structural features does
not yet succeed; I return to this as an open problem in
Section~\ref{subsec:disc-heterogeneous}. For reference, under feature shift (target
accuracy $0.61$): Uncalibrated ECE $0.203$, Source-TS $0.191$, STAC-global $0.100$,
Oracle $0.024$.

\subsection{Real graphs across the homophily spectrum}\label{subsec:exp-real}
I evaluate on five real graphs~\citep{platonov2023critical}, namely roman-empire
($h{=}0.05,K{=}18$), amazon-ratings ($0.38,5$), tolokers ($0.59,2$), minesweeper
($0.68,2$) and questions ($0.84,2$), under a fixed relative covariate shift
(Fig.~\ref{fig:real}, Table~\ref{tab:real}). The robust finding matches the theory:
a single \emph{oracle} temperature recovers calibration on \emph{every} graph
(ECE $0.015$--$0.024$), across the full homophily range and multi-class settings. However
the \emph{label-free} STAC is not reliable here: because the accuracy estimate overshoots
under shift (e.g.\ amazon-ratings $\hat a{=}0.81$ vs.\ true $0.38$), it can worsen ECE
($0.085\!\to\!0.432$). In-distribution calibrators
\citep{wang2021cagcn,hsu2022gats,zhuang2025gets,li2025wats,tang2024simcalib} are
source-fitted and thus inherit the failure of source-TS under shift; empirical
uncertainty-quantification methods designed with shift in mind, such as
G-$\Delta$UQ~\citep{trivedi2024gduq}, could plausibly supply a better $\hat a$ here, since
neither it nor STAC's simple ATC/disagreement estimator was designed for the
heterophilous, high-$K$ regime of roman-empire; I see this as a promising, direct avenue
for future work (Section~\ref{subsec:disc-crux}) rather than a competing claim. I report
STAC's negative result plainly.

\begin{figure}[t]
\centering
\includegraphics[width=\linewidth]{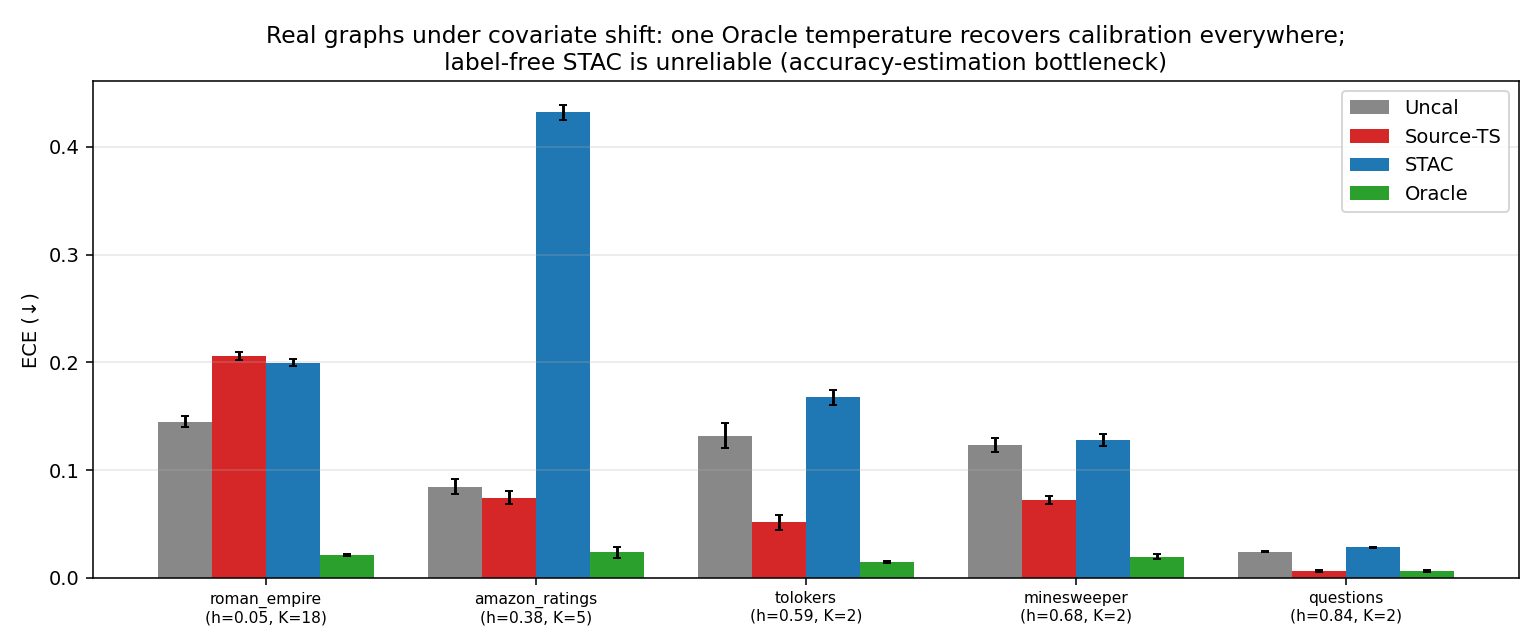}
\caption{\textbf{Five real graphs under covariate shift.} One oracle temperature recovers
calibration across the homophily spectrum; STAC is unreliable
(accuracy-estimation bottleneck).}
\label{fig:real}
\end{figure}

\begin{table}[t]
\centering\small
\caption{ECE ($\downarrow$) on real graphs under covariate shift (3 splits).}
\label{tab:real}
\begin{tabular}{lccccc}
\toprule
graph & $h$ & Uncal. & Src-$T$ & STAC & Oracle\\
\midrule
roman-empire   & 0.05 & 0.145 & 0.206 & 0.200 & \textbf{0.022}\\
amazon-ratings & 0.38 & 0.085 & 0.075 & 0.432 & \textbf{0.024}\\
tolokers       & 0.59 & 0.132 & 0.052 & 0.168 & \textbf{0.015}\\
minesweeper    & 0.68 & 0.123 & 0.072 & 0.128 & \textbf{0.020}\\
questions      & 0.84 & 0.024 & 0.006 & 0.028 & \textbf{0.007}\\
\bottomrule
\end{tabular}
\end{table}

\subsection{Covariate-shift strength on real graphs}\label{subsec:exp-real-cov}
Proposition~\ref{prop:cov} predicts that covariate shift makes the frozen model
progressively over-confident, and \emph{only} over-confident. Sweeping the relative
feature-noise strength $\gamma$ on all five graphs, the confidence$-$accuracy gap grows
monotonically with $\gamma$ and never changes sign (Fig.~\ref{fig:sweep}a; e.g.\
roman-empire $-0.07\!\to\!+0.57$, minesweeper $+0.07\!\to\!+0.29$), while a single oracle
temperature keeps ECE low throughout (Fig.~\ref{fig:sweep}b). The nearly flat curve for
questions ($h{=}0.84$, accuracy $0.97$) is expected: Proposition~\ref{prop:cov} gives a
gap $\propto\gamma/(4h(1{-}h)\rho{+}1)$, small when accuracy is high.

\begin{figure}[t]
\centering
\includegraphics[width=\linewidth]{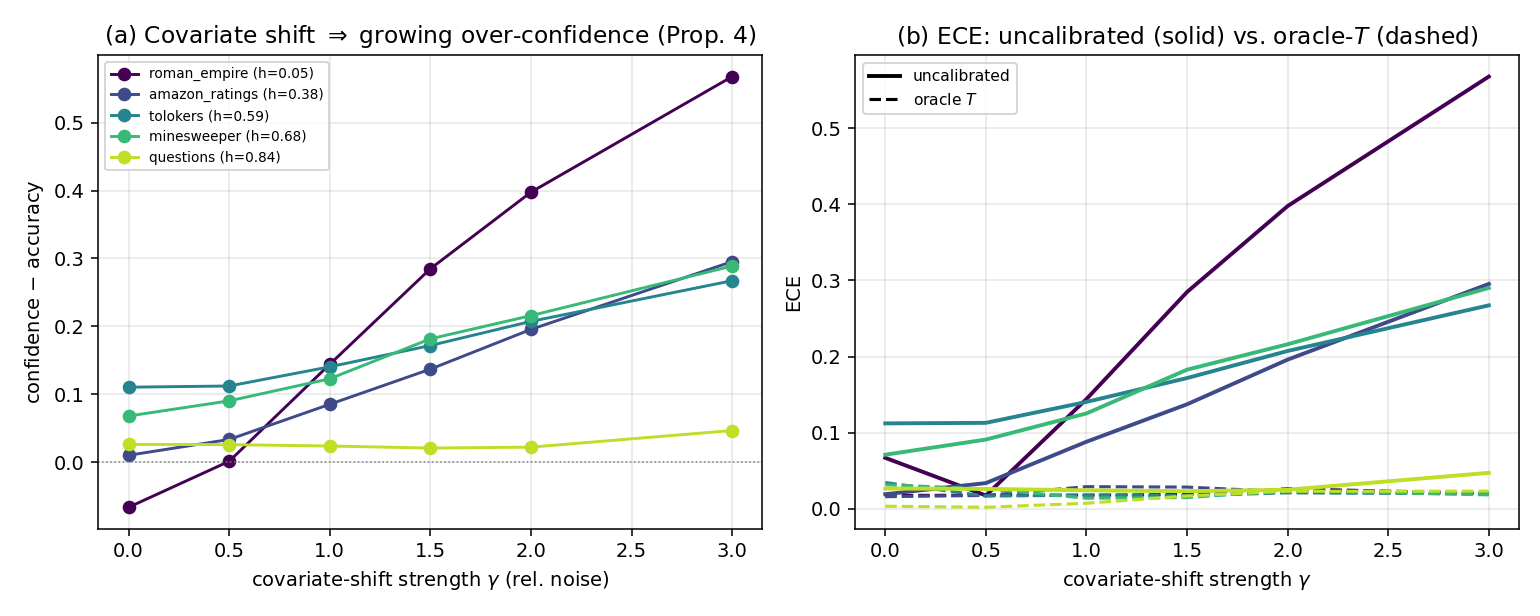}
\caption{\textbf{Covariate-shift sweep (Proposition~\ref{prop:cov}).} (a) over-confidence
grows with shift strength on every real graph; (b) uncalibrated ECE (solid) rises while a
single oracle temperature (dashed) stays low.}
\label{fig:sweep}
\end{figure}

\subsection{A different shift type: structural rewiring}
To test breadth beyond covariate noise I rewire a fraction of the target edges. This
miscalibrates several graphs (e.g.\ minesweeper ECE $0.07\!\to\!0.14$ at $75\%$ rewiring),
and a single oracle temperature again recovers calibration across all five
(ECE $\le0.03$; Fig.~\ref{fig:struct}). For class-imbalanced graphs whose random-edge
homophily is already high (questions, tolokers) rewiring changes homophily little and
calibration is largely unaffected, consistent with the theory, in which the effect is
mediated by the change in edge homophily.

\begin{figure}[t]
\centering
\includegraphics[width=0.86\linewidth]{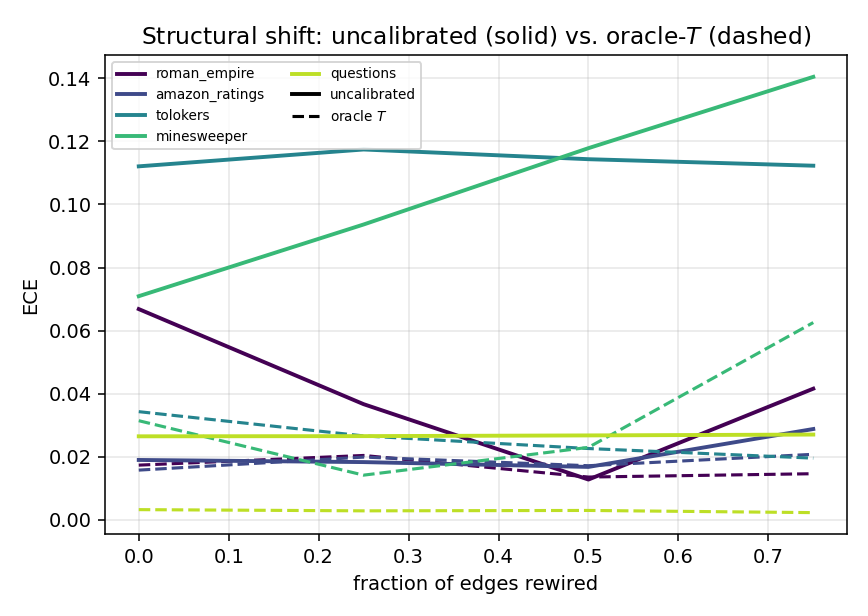}
\caption{\textbf{Structural shift.} Uncalibrated ECE (solid) vs.\ oracle temperature
(dashed) as target edges are rewired; one temperature recovers calibration across graphs.}
\label{fig:struct}
\end{figure}

\section{Discussion, limitations, and open problems}\label{sec:discussion}
I merge discussion and limitations deliberately: every limitation below is the direct,
falsifiable counterpart of a claim made in Section~\ref{sec:theory}, and reading them
together is what makes the theory's scope honest rather than merely optimistic.

\subsection{The calibration slope as a diagnostic}
Beyond motivating STAC, $\kap$ is useful on its own: Eq.~\eqref{eq:kappa} can be
evaluated from an estimate of $(h_s,h_t,\rho)$ alone, quantities that are far easier to
estimate without labels than accuracy itself (homophily can be estimated from a
label-propagation proxy or from the model's own agreement statistics; SNR from feature
clustering), to \emph{predict} whether a deployed GNN is likely to be over- or
under-confident on a new graph before any recalibration is attempted at all, using
Proposition~\ref{prop:bound} as a conservative severity bound.

\subsection{Why source-fitted calibrators fail under shift}\label{subsec:disc-sourcefit}
Every in-distribution calibrator surveyed in Section~\ref{sec:related} (temperature
scaling itself, CaGCN, GATS, GETS, WATS, SimCalib) fits its correction (a scalar or a
per-node function) by minimizing NLL on a labeled split assumed to represent the
distribution the model will see at test time. Theorem~\ref{thm:kappa} makes explicit
\emph{why} this assumption breaking is fatal: the fitted correction is optimal for
$\kap(h_s,h_s,\rho)=1$ by construction, but the deployment-time slope is
$\kap(h_s,h_t,\rho)\ne1$ whenever $h_t\ne h_s$, and no amount of additional labeled
\emph{source} data changes this, because the source split never observes $h_t$. This is
precisely the mechanism Fig.~\ref{fig:mech} makes visible: source-TS's assumption curve
is flat by construction, while the true target accuracy moves. Methods with a per-node
correction do not escape this: a per-node function fit on source data still targets
$\kap(h_s,h_s,\rho)=1$ pointwise, uniformly across nodes, so it fails in exactly the same
direction as global temperature scaling, merely with additional, source-fitted degrees of
freedom that provide no protection against a shift in $h$.

\subsection{When does per-node calibration help?}
Corollary~\ref{cor:global} localizes the value of per-node calibration precisely: it
helps if and only if different nodes experience \emph{different} effective slopes $\kap_i$
(heterogeneous shift), never as a matter of raw modeling capacity. This reframes the
comparison in Section~\ref{sec:related} between global and per-node calibrators: the
question is not ``does a more expressive per-node model calibrate better'' but ``is the
shift itself homogeneous or heterogeneous'', which is a property of the deployment
scenario, not of the calibrator. My heterogeneous-shift ablation
(Section~\ref{subsec:exp-method}) is a first attempt to answer this operationally, and it
suggests that answering it well requires per-node \emph{shift-detection} signals that are
more informative than the structural features I tried; see
Section~\ref{subsec:disc-heterogeneous}.

\subsection{The crux is label-free accuracy estimation}\label{subsec:disc-crux}
This is the paper's central limitation, and I state it as plainly as possible: STAC's
entire practical reliability rests on $\hat a$, and Section~\ref{subsec:exp-real} shows
that on real, heterophilous, high-$K$ graphs (roman-empire, amazon-ratings), both ATC and
disagreement-based estimators can be badly wrong, overshooting true accuracy by a wide
margin, and thereby set $\Tstar$ in the wrong direction. Two implications follow, one
constructive and one methodological. Constructively, the oracle results
(Table~\ref{tab:real}) show the theory's ceiling is high: \emph{if} $\hat a$ were
accurate, a single temperature would suffice everywhere I tested, from $h{=}0.05$ to
$h{=}0.84$ and up to $K{=}18$ classes, so the return on improving label-free accuracy
estimation specifically \emph{for graphs} is large. Methodologically, ATC and
disagreement were designed and validated primarily on covariate shift for
i.i.d.\ image/text data~\citep{garg2022atc,jiang2022disagreement}; graphs add at least two
confounds absent there: aggregation compounds shift across hops, so a node's effective
shift depends on its neighbourhood's shift, not only its own features, and homophily
itself is a further axis of shift that has no counterpart in the i.i.d.\ setting these
estimators were built for. A concrete, falsifiable next step is a graph-native accuracy
estimator that conditions disagreement on the local homophily estimated at the same node,
or that combines STAC's estimator with the epistemic-uncertainty signal of
G-$\Delta$UQ~\citep{trivedi2024gduq} (Section~\ref{sec:related}), which was explicitly
validated under graph covariate and concept shift and could plausibly correct exactly the
overshoot I observe on amazon-ratings and roman-empire.

\subsection{Scope of the theory}\label{subsec:disc-scope}
The closed forms assume Gaussian class-conditional features and are asymptotic in degree
$d$ (a Berry--Esseen-type correction is $O(d^{-1/2})$); Section~\ref{sec:exp} shows this
is already accurate at moderate degree ($d\!\approx\!20$) but I do not claim a
non-asymptotic, finite-degree guarantee. The CSBM is also a \emph{fixed-homophily} model:
real graphs mix a global homophily with substantial node-to-node heterogeneity, which is
exactly the gap between the homogeneous-shift theory of Section~\ref{sec:theory} and the
heterogeneous-shift ablation of Section~\ref{subsec:exp-method}. Finally, near chance
accuracy (as on parts of roman-empire and amazon-ratings) the logistic approximation
underlying Theorem~\ref{thm:kappa} is least informative, precisely the regime where
label-free accuracy estimation is also hardest (Section~\ref{subsec:disc-crux}); these
two limitations compound rather than being independent, which is why real-graph
performance in Table~\ref{tab:real} is weakest exactly where the pointwise theory is
weakest.

\subsection{Heterogeneous shift is the open per-node problem}\label{subsec:disc-heterogeneous}
Corollary~\ref{cor:global} says \emph{when} per-node calibration is needed
(heterogeneous shift) but not \emph{how} to detect, node by node, which effective $\kap_i$
applies; my attempt (Section~\ref{subsec:exp-method}) with generic structural
perturbation signals moved error-detection AUROC only marginally above a
structure-free baseline. I view this negative result as informative rather than
discouraging: it narrows the open problem to finding per-node signals that specifically
track \emph{local} shift severity; candidates include local homophily estimated from a
node's immediate neighbourhood (directly motivated by Eq.~\eqref{eq:kappa}), a per-node
analogue of the epistemic-uncertainty estimator of Trivedi et al.~\citep{trivedi2024gduq}, or the
per-node conformal width produced by Huang et al.~\citep{huang2023conformal}, which is a different
notion of per-node reliability (set size) that could plausibly correlate with $\kap_i$
without requiring labels.

\subsection{Broader outlook}
I see three natural combinations with the empirical literature surveyed in
Section~\ref{sec:related}. First, G-$\Delta$UQ's per-model epistemic-uncertainty
estimate~\citep{trivedi2024gduq} could serve directly as the per-node shift-severity
signal that Section~\ref{subsec:disc-heterogeneous} shows is currently missing. Second,
the distribution-free coverage guarantees of conformalized
GNNs~\citep{huang2023conformal} address a set-valued, not pointwise, notion of
reliability; combining a globally-corrected temperature with a conformal wrapper is
plausible and would let the two guarantees (closed-form pointwise calibration and
distribution-free coverage) compose rather than compete. Third, GDA-SpecReg's
risk bound~\citep{you2023gda} bounds target \emph{error}; an accuracy bound of that kind,
if tight enough, is itself a candidate label-free $\hat a$ for STAC, directly addressing
Section~\ref{subsec:disc-crux}.

\section{Conclusion}\label{sec:conclusion}
This paper gives, to my knowledge, the first closed-form account of how distribution
shift moves the confidence-accuracy relationship of a graph neural network: a single
calibration slope $\kap$, explicit in source/target homophily and feature
signal-to-noise ratio, determines simultaneously the \emph{direction} of miscalibration
($\kap<1$ over-confident, $\kap>1$ under-confident), its \emph{magnitude} (via the proven
ECE bound), and its \emph{exact correction} ($\Tstar=1/\kap$, provably optimal at the
population level and provably sufficient as a single global scalar under homogeneous
shift). The theory is not a qualitative story: it predicts the measured slope and the
oracle temperature quantitatively in simulation ($r=0.99$), and its central operational
consequence (a single temperature recovers calibration) holds on five real graphs
spanning two orders of magnitude in homophily and up to $18$ classes.

Turning the theory into a fully automatic, label-free recalibrator reduces to estimating
one number, target accuracy, from unlabeled data; I showed this reduction is correct in
that a good estimate (the oracle) suffices, but that today's off-the-shelf estimators
(ATC, disagreement) are not yet accurate enough on real, heterophilous graphs to realize
the theory's full promise. I see this not as a weakness of the theory but as its most
useful output: it isolates exactly the sub-problem (shift-robust, graph-native accuracy
estimation) on which future work should concentrate to convert closed-form understanding
into a deployable, zero-label GNN calibrator, and Section~\ref{sec:discussion} lays out
concrete, falsifiable directions toward that goal, including combinations with recent
empirical uncertainty-quantification and conformal-prediction methods for GNNs. More
broadly, I hope the calibration-slope viewpoint offers a template for other
structure-dependent reliability questions on graphs: write down the smallest model that
reproduces a failure mode in closed form, derive its falsifiable consequences, and test
how far its predictions travel toward real, deployed graph learning systems.

\appendix
\section{Assumptions and proofs}\label{app:proofs}
This appendix states the formal assumptions underlying Sections~\ref{sec:theory}, and
gives complete proofs of every theoretical claim in the main text. All results concern a
linear GNN on a CSBM in a large-degree regime; finite-degree accuracy is discussed in
Remark~\ref{rem:berry} and confirmed numerically in Section~\ref{sec:exp}
(Figs.~\ref{fig:val}--\ref{fig:ext}).

\subsection{Model and assumptions}
\begin{assumption}[CSBM]\label{as:csbm}
There are $N$ nodes and $K$ balanced classes with labels $y_i$ i.i.d.\ uniform on
$\{1,\dots,K\}$. Node features are $x_i=\mu_{y_i}+\xi_i$, with
$\xi_i\sim\mathcal N(0,\sigma^2 I_F)$ i.i.d.\ and class means satisfying $\|\mu_c\|=r$ for
all $c$ and $\sum_c\mu_c=0$ (for $K{=}2$, $\mu_+=r\hat\mu$, $\mu_-=-r\hat\mu$ for a unit
vector $\hat\mu$). Edges are drawn independently: for $i\ne j$, $\Pr(i\sim j)=p$ if
$y_i=y_j$ and $q$ otherwise. Write the edge homophily $h=p/(p+q(K{-}1))$ and the
signal-to-noise ratio $\rho=r^2/\sigma^2$.
\end{assumption}

\begin{assumption}[Large-degree regime]\label{as:deg}
The expected degree $d=p(\tfrac{N}{K}{-}1)+q\tfrac{N(K-1)}{K}$ satisfies $d\to\infty$ as
$N\to\infty$, with $h$ and $\rho$ held fixed.
\end{assumption}

\begin{assumption}[Source-calibrated Bayes-direction classifier]\label{as:clf}
The classifier is linear with weight along the Bayes-optimal direction ($\hat\mu$ for
$K{=}2$; the class-mean directions for $K{>}2$) and zero bias, optimal for the symmetric
Gaussian model of Assumption~\ref{as:csbm}. Its scalar scale $a>0$ is fixed by requiring
calibration on the \emph{source} graph, i.e.\ $\kap(h_s)=1$. A consistent estimator such
as regularization-free logistic regression recovers this direction and, after source
temperature scaling, this scale.
\end{assumption}

\begin{assumption}[Regular-degree approximation, self-loop operator only]\label{as:reg}
For Proposition~\ref{prop:sl} I approximate all degrees by the expected degree $d$ in
the symmetric normalization $\hat A_{ij}=((d_i{+}1)(d_j{+}1))^{-1/2}\approx(d{+}1)^{-1}$.
The error is $O(\mathrm{Var}(d_i)/d^2)$ and vanishes under Assumption~\ref{as:deg} for
concentrated degrees.
\end{assumption}

Throughout, $\soft(u)=(1+e^{-u})^{-1}$ and $\varphi$ is the standard normal density.

\subsection{Proof of Lemma~\ref{lem:agg} (aggregated coordinate)}
\begin{proof}
For a neighbour $j$, $\hat\mu^\top x_j=\hat\mu^\top\mu_{y_j}+\hat\mu^\top\xi_j=y_j r+\zeta_j$
with $\zeta_j\sim\mathcal N(0,\sigma^2)$ i.i.d.\ (since $\hat\mu^\top\mu_\pm=\pm r$). Hence
$g_i=\frac1{d_i}\big(r\sum_{j\sim i}y_j+\sum_{j\sim i}\zeta_j\big)$. Given $y_i{=}{+}1$,
each neighbour is same-class ($y_j{=}{+}1$) with probability $h$ and opposite
($y_j{=}{-}1$) with probability $1{-}h$, independently; thus with
$k=\#\{j:y_j{=}{+}1\}\sim\mathrm{Binomial}(d_i,h)$ I have $\sum_j y_j=2k-d_i$, so
$\E[\sum_j y_j]=d_i(2h{-}1)$ and $\mathrm{Var}[\sum_j y_j]=4\,\mathrm{Var}(k)=4d_ih(1{-}h)$.
The noise sum has mean $0$ and variance $d_i\sigma^2$ and is independent of the labels.
Dividing by $d_i$ gives the stated mean and variance exactly. Asymptotic normality
follows from the Lindeberg CLT applied to the independent bounded label terms and the
Gaussian noise term.
\end{proof}

\begin{remark}[Finite-degree error]\label{rem:berry}
By the Berry--Esseen theorem the CDF of $g_i$ differs from its Gaussian limit by
$O(d_i^{-1/2})$. Empirically the closed forms hold to a few percent already at
$d\approx20$ (Section~\ref{sec:exp}); accuracy degrades only when target accuracy
approaches $\tfrac12$ or $1$, where the calibration slope is inherently ill-conditioned.
\end{remark}

\subsection{Proof of Theorem~\ref{thm:kappa} and Corollary~\ref{cor:global}}
\begin{proof}
Take $t>0$ (the case $t<0$ is symmetric), so the prediction is $+1$ and the event
``correct'' is $\{y{=}{+}1\}$. With balanced classes and the two Gaussian likelihoods,
\[
\Pr(y{=}{+}1\mid\delta{=}t)=\frac{\varphi\!\big(\tfrac{t-m}{s}\big)}
{\varphi\!\big(\tfrac{t-m}{s}\big)+\varphi\!\big(\tfrac{t+m}{s}\big)}
=\frac1{1+\exp\!\big(-\tfrac{(t+m)^2-(t-m)^2}{2s^2}\big)}
=\frac1{1+e^{-2mt/s^2}}=\soft(\kap t).
\]
The model outputs $\soft(\delta)$ as the probability of class $+1$; for a prediction at
$\delta=t>0$ its confidence is $\soft(t)$. Calibration requires $\soft(t)=\soft(\kap t)$
for all $t$, i.e.\ $\kap=1$; the sign of $\kap-1$ gives over-/under-confidence. Applying
temperature $T$ maps confidence to $\soft(t/T)$, which equals the true posterior
$\soft(\kap t)$ for all $t$ iff $1/T=\kap$. Because the temperatured model then coincides
with the true conditional distribution, it minimizes the population NLL; hence
$\Tstar=1/\kap$. Finally,
\[
\kap=\frac{2m}{s^2}=\frac{2a(2h{-}1)r}{a^2(4h(1{-}h)r^2+\sigma^2)/d}
=\frac{2d(2h{-}1)r}{a(4h(1{-}h)r^2+\sigma^2)}.
\]
Source calibration $\kap(h_s)=1$ gives $a=\dfrac{2d(2h_s{-}1)r}{4h_s(1{-}h_s)r^2+\sigma^2}$;
substituting into $\kap(h_t)$, the common factors $2dr$ cancel, and dividing numerator and
denominator by $\sigma^2$ (so $r^2/\sigma^2=\rho$) yields Eq.~\eqref{eq:kappa}. Degree
cancels.

For Corollary~\ref{cor:global}: under homogeneous shift, $\Tstar=1/\kap$ makes the model
equal to the true conditional, achieving the minimum of population NLL and zero
population calibration error, by the argument above applied pointwise. Population NLL as
a function of a per-node scaling $T_i$ is, at fixed $\delta_i$, strictly convex in $1/T_i$
with a unique minimizer $1/\Tstar$; therefore any assignment $\{T_i\}$ with $T_i\ne\Tstar$
on a positive-measure set strictly increases population NLL relative to the constant
assignment $T_i\equiv\Tstar$. Thus per-node temperatures cannot improve calibration under
a homogeneous $\kap$; they can only help when $\kap$ varies across nodes (heterogeneous
shift).
\end{proof}

\subsection{Proof of Proposition~\ref{prop:sl} (self-loop GCN)}
\begin{proof}
Under Assumption~\ref{as:reg}, $z_i=\frac1{d+1}\big(x_i+\sum_{j\sim i}x_j\big)$.
Projecting on $\hat\mu$ and conditioning on $y_i{=}{+}1$: the self term contributes
$r+\zeta_i$ (same-class, deterministic mean $+r$); the $d$ neighbour terms contribute as
in Lemma~\ref{lem:agg}. Hence
$\E[\hat\mu^\top z_i]=\frac{r}{d+1}(1+d(2h{-}1))=\frac{rA(h)}{d+1}$, and
$\mathrm{Var}[\hat\mu^\top z_i]=\frac1{(d+1)^2}\big(4dr^2h(1{-}h)+(d{+}1)\sigma^2\big)
=\frac{\sigma^2B(h)}{(d+1)^2}$ (the self term adds one unit of noise variance but no label
variance). With $\delta=a\,\hat\mu^\top z_i$ I get
$\kap_{SL}(h)=\frac{2m}{s^2}=\frac{2arA(h)/(d{+}1)}{a^2\sigma^2B(h)/(d{+}1)^2}
=\frac{2r(d{+}1)A(h)}{a\sigma^2B(h)}$. Fixing $a$ by $\kap_{SL}(h_s)=1$ and forming the
ratio, all factors except $A$ and $B$ cancel, giving the claim. As $d\to\infty$,
$A(h)\sim d(2h{-}1)$ and $B(h)\sim d(4h(1{-}h)\rho{+}1)$, recovering Theorem~\ref{thm:kappa}.
\end{proof}

\subsection{Proof of Proposition~\ref{prop:K} ($K$ classes)}
\begin{proof}
A neighbour is same-class (probability $h$), contributing $\mu_c$, or different-class
(probability $1{-}h$) drawn uniformly among the $K{-}1$ other classes, with conditional
mean $\frac1{K-1}\sum_{c'\ne c}\mu_{c'}=-\frac{\mu_c}{K-1}$ because $\sum_{c'}\mu_{c'}=0$.
Hence
\[
\E[z_i\mid y_i{=}c]=h\mu_c+(1{-}h)\Big({-}\frac{\mu_c}{K-1}\Big)
=\mu_c\,\frac{h(K{-}1)-(1{-}h)}{K-1}=\mu_c\,\frac{hK-1}{K-1}.
\]
Since this scaling is common to all class directions, the source-trained logits scale by
$c_K(h_t)/c_K(h_s)$ under shift.
\end{proof}
\begin{remark}
The exact multiclass calibration slope requires the order statistics of the softmax and
is not needed here: Corollary~\ref{cor:global}'s mechanism (uniform logit rescaling
$\Rightarrow$ a single temperature suffices) carries over, and is confirmed numerically
for $K{=}3,4,5$ (Fig.~\ref{fig:ext}).
\end{remark}

\subsection{Proof of Proposition~\ref{prop:cov} (covariate shift)}
\begin{proof}
Extra zero-mean feature noise leaves $m$ unchanged and increases the logit variance to
$s^2(\tau)=a^2(4h(1{-}h)r^2+\sigma^2+\tau^2)/d$. Thus
$\kap_{\mathrm{cov}}=2m/s^2(\tau)=\kap(0)\cdot s^2(0)/s^2(\tau)$; with $\kap(0)=1$ (source
calibration) and dividing by $\sigma^2$, the formula follows, and
$\Tstar=1/\kap_{\mathrm{cov}}$.
\end{proof}

\subsection{Proof of Proposition~\ref{prop:bound} (ECE bound)}
\begin{proof}
Population ECE $=\E_\delta\big[|\Pr(\text{correct}\mid\delta)-\text{conf}(\delta)|\big]$.
By Theorem~\ref{thm:kappa}, at $\delta=t$ this integrand is $|\soft(\kap t)-\soft(t)|$
(with $t=|\delta|$ the predicted confidence's logit). Since $\soft$ is
$\tfrac14$-Lipschitz, $|\soft(\kap t)-\soft(t)|\le\tfrac14|\kap t-t|=\tfrac14|\kap-1||t|$.
Taking expectation over $\delta$ gives the bound. It vanishes as $\kap\to1$ (no shift) and
grows with the logit magnitude.
\end{proof}

\subsection{Discussion of scope}
The results are exact in the large-degree Gaussian limit (Assumption~\ref{as:deg}) with a
Bayes-direction, source-calibrated linear classifier (Assumption~\ref{as:clf}). Two
honest caveats, expanded operationally in Section~\ref{subsec:disc-scope}: \emph{(i)}
finite-degree corrections are $O(d^{-1/2})$ (Remark~\ref{rem:berry}), negligible in the
moderate-accuracy regime but growing as accuracy approaches $\tfrac12$ or $1$;
\emph{(ii)} for deep nonlinear GNNs the closed forms hold only qualitatively, though the
operational consequence (a single global temperature suffices under homogeneous
shift) is confirmed on real graphs in Section~\ref{subsec:exp-real}
(e.g.\ amazon-ratings ECE $0.085\!\to\!0.024$ after one oracle temperature).

\bibliographystyle{elsarticle-num}
\bibliography{references}

\end{document}